\documentclass[wcp]{jmlr}

\pdfoutput = 1

\usepackage{hyperref}
\usepackage{url}
\usepackage{rotating}
\usepackage{longtable}
\usepackage{booktabs}
\usepackage[load-configurations=version-1]{siunitx} 
\usepackage{siunitx}

\usepackage{enumerate}
\usepackage{booktabs}
\input{header}

\jmlrvolume{29}
\jmlryear{2013}
\jmlrworkshop{ACML 2013}

\title[\uma]{Unconfused Ultraconservative Multiclass Algorithms}

 \author{\Name{Ugo Louche} \Email{ugo.louche@lif.univ-mrs.fr} \AND
 \Name{Liva Ralaivola} \Email{liva.ralaivola@lif.univ-mrs.fr} \\
 \addr Qarma, Lab. d'Informatique Fondamentale de
 Marseille, CNRS, Aix-Marseille University, France}

\editor{Cheng Soon Ong and Tu Bao Ho}

\begin{document}

\maketitle

\begin{abstract}

We tackle the problem of learning linear
classifiers from noisy datasets in a multiclass setting. The two-class version
of this problem was
studied a few years ago by,
e.g. \citet{bylander94learning} and 
\citet{blum96polynomialtime}: in these contributions, the proposed
approaches to fight the noise
revolve around a Perceptron learning scheme fed with peculiar
examples computed through a weighted average of points from the
noisy training set.
We propose to build upon these approaches and we introduce a new
algorithm called \uma (for Unconfused Multiclass additive Algorithm)
which may be seen as a generalization to the multiclass
setting of the previous approaches. 
In order to characterize the noise we use the
{\em confusion matrix} as a 
multiclass extension of the classification noise studied in the
aforementioned literature.
Theoretically well-founded, \uma furthermore displays very good
empirical noise robustness, as evidenced by numerical simulations
conducted on both synthetic  and real data.


\end{abstract}


\section{Introduction} \label{sec:Intro}

\paragraph{Context.} 
This paper deals with linear multiclass classification problems defined on an
input space $\inputspace$ (e.g., $\inputspace=\realset^d$) and a set
of classes $$\classes\doteq\{1,\ldots Q\}.$$ In particular, we are
interested in  establishing the robustness of {\em
ultraconservative additive} algorithms \citep{crammer03ultraconservative} to label noise classification in
the multiclass setting ---in order to lighten notation, we will now refer to these algorithms as {\em ultraconservative algorithms}. We study whether it  is possible to learn a linear predictor from a training set  
$$\trainingset\doteq\{(\bfx_i,y_i)\}_{i=1}^n$$
where $y_i\in\classes$ is a corrupted version of a {\em true}, i.e. deterministically
computed class, $t(\bfx_i)\in\classes$ associated with $\bfx_i$,
according to some {\em concept} $t$. The random noise process $Y$ that corrupts the label to provides the $y_i$'s given the $\V{x}_i$'s is
fully described by a {\em confusion matrix}
$\confusion=(\confusion_{pq})_{p,q}\in\realset^{Q\times Q}$ so that
$$\forall \V{x},\;\proba_{Y}(Y=p|\V{x})=\confusion_{pt(\V{x})}.$$
 The goal that we would like to achieve is to
provide a learning procedure able to deal with the {\em confusion
  noise} present in the  training set
$\trainingset$ to give rise  to a classifier $h$ with small risk
$\proba_{X\sim\distribution}(h(X)\neq t(X))$ ---$\distribution$
being the distribution of the $\bfx_i$'s. As we want to recover
from the confusion noise, we use the term {\em unconfused} to
characterize the procedures we propose.

\citet{crammer03ultraconservative} introduce ultraconservative
online learning
algorithms, which output multiclass linear predictors of the form $$f(\V{x}) =
\argmax_r \langle \V{w}_r, \V{x} \rangle.$$ When processing a training
pair $(\V{x},y)$, these procedures perform updates of the
form
$$\V{w}_q^{\text{new}}\leftarrow\V{w}_q+\tau_q\V{x}, \; q=1,\ldots Q,$$
so that:
\begin{itemize}
\item if $y=\argmax_q\langle\V{w}_q,\bfx\rangle$, then
  $\tau_1=\cdots=\tau_Q=0$ (no update is made);
\item otherwise (i.e. $y\neq\argmax_q\langle\V{w}_q,\bfx\rangle$),
  then, given ${\cal E}\doteq\{r:r\neq y,
  \langle \V{w}_r,\V{x}\rangle \geq \langle \V{w}_y,\V{x}\rangle\}$,
  the $\tau_q$'s verify: a) $\tau_y=1$, b) $\forall
  q\in{\cal E}, \tau_q\leq 0$, c) $\forall
  q\not\in\mathcal{E}\union\{y\}, \tau_q=0$, and d) $\sum_q^{Q}\tau_q=0$. 
\end{itemize}
Ultraconservative learning procedures, have very nice
theoretical properties regarding their convergence in the case of
linearly separable datasets, provided a sufficient separation {\em
  margin} is guaranteed (as formalized in Assumption~\ref{ass:margin}
below). In turn, these convergence-related properties
yield generalization guarantees about the quality of the predictor learned. We build upon these nice convergence
properties to show that ultraconservative algorithms are robust to a confusion noise process,
provided an access to the confusion matrix  ${\cal C}$ is granted and this paper is essentially
devoted to proving how/why  ultraconservative multiclass algorithms are
indeed robust to such situations.
To some extent, the results provided in the present contribution may be viewed as a generalization of the contributions on
learning binary perceptrons under misclassification noise \citep{blum96polynomialtime,bylander94learning}.

Besides the theoretical questions raised by the learning setting
considered, we may depict the following example of an actual learning
scenario where learning from noisy data is relevant. This learning
scenario will be further investigated from an empirical standpoint in
the section devoted to numerical simulations (Section~\ref{sec:Expe}).
\begin{example}
\label{ex:redwire}
One  situation where coping with mislabelled data is required
arises in scenarios where labelling data is very expensive. Imagine a
task of text categorization from a training set
$\trainingset=\trainingset_{\ell}\union\trainingset_u$, where 
$\trainingset_{\ell}=\{(\bfx_i,y_i)\}_{i=1}^n$ is a set of $n$
labelled training examples and $\trainingset_u=\{\bfx_{n+i}\}_{i=1}^m$ is
a set of $m$ unlabelled vectors; in order to fall back to a realistic
training scenario, we may assume that $n<<m$. A possible three-stage strategy to
learn a predictor is as follows: first learn a predictor $f_{\ell}$ on
$\trainingset_{\ell}$ and estimate its confusion $\confusion$ error {\em via} a
cross-validation procedure, second, use the learned predictor to label all the data
in $\trainingset_u$ to produce
the labelled traning set $\widehat{\trainingset}=\{(\bfx_{n+i},t_{n+i}:=f(\bfx_{n+i}))\}_{i=1}^m$ and finally,
learn a classifier $f$ from $\widehat{\trainingset}$ {\em and} the
confusion information $\confusion$.
\end{example}

\paragraph{Contributions}
Our main contribution is to show that it is both 
practically and theoretically possible to learn a
multiclass classifier on noisy data as long as some 
information on the noise process is available.
We propose a way to compute update vectors for
any ultraconservative algorithm which
allows us to handle massive amount of mislabeled 
data without consequent loss of accuracy.
Moreover, we provide a thorough analysis of
our method and show that the strong
theoretical guarantees that caracterize
the family of ultraconservative algorithm carry over to the noisy scenario.

\paragraph{Organization of the paper.} 
Section \ref{sec:setting} formally states the setting
we consider throughout this paper. Section~\ref{sec:uma}
provides the details of our main contribution: the \uma
algorithm. Section~\ref{sec:Expe}
presents numerical simulation that support the soundness of our approach.



\section{Setting and Problem} \label{sec:setting}
\subsection{Noisy Labels with Underlying Linear Concept}
The probabilistic setting we consider hinges on the existence of two
components. On the one hand, we assume an
unknown (but fixed) probability distribution $\UKDistri$ on
the {\em intput space} $\inputspace\doteq\realset^d$;   without loss of generality, we suppose that
 $\proba_{X\sim\UKDistri}(\|X\|=1)=1$,
where $\|\cdot\|$ is the Euclidean norm. On the other hand, we
also assume the existence of a deterministic labelling function
$t:\inputspace\to\classes$, where $\classes\doteq\{1,\ldots Q\}$, which
associates a label $t(\bfx)$ to any input example $\bfx$; in the
{\em Probably Approximately Correct} literature, $t$ is sometimes  referred
to as a {\em concept}
\citep{kearns94introduction,valiant84theory}. Throughout, we assume the
following:
\begin{assumption}[Linear Separability with $\margin$ Margin.]
\label{ass:margin}
Concept $t$
is such that there exists a {\em compatible linear} classifier $f^*$ with margin
$\margin>0$. This means that 
\begin{equation}
\proba_{X\sim\UKDistri}(f^*(X)\neq t(X))= 0,
\tag{comptability of $f^*$ wrt $t$}
\end{equation}
and  there exist $\margin>0$ and  $W^*=[\bfw_1^*\cdots\bfw_Q^*]\in\realset^{d\times
  Q}$ such that 
\begin{align}
&\forall\bfx\in\inputspace,\; f^{*}(\bfx)=\argmax\nolimits_{q\in\classes}\langle\bfw_q,\bfx\rangle,\tag{$f^*$
  is a linear classifier}\\
&\proba_{X\sim \UKDistri}\left\{\exists p\neq t(X):\left\langle\bfw^*_{t(X)}-\bfw^*_{p},X\right\rangle\leq\margin\right\}=0.\tag{$f^*$
has margin $\margin$ wrt $t$ and $\UKDistri$}
\end{align}
(Here, $\langle\cdot,\cdot\rangle$ denotes the canonical inner product
of $\realset^d$.)
\label{ass:linear}
\end{assumption}

In a usual setting, one would be asked to learn a
classifier $f$ from a training set
$$\inputset_{\text{true}}\doteq\{(\V{x}_i,t(\V{x}_i)\}_{i=1}^n$$ made of $n$
labelled pairs from $\inputspace\times\classes$ such that the
$\V{x}_i$'s are independent realizations of a random variable $X$
distributed according to $\UKDistri$, with the objective of minimizing
the \emph{true risk} or {\em misclassification
  error} $\risk_{\terror}(f)$ of $f$ given by
\begin{align} 
  \risk_{\terror}(f) \doteq \proba_{X\sim\UKDistri}(f(X) \neq t(X)).
\end{align}
In other words, the objective is for $f$ to have a prediction behavior
as close as possible to that of $t$.

As announced in the introduction, there is a little twist in the problem
that we are going to tackle. Instead of having direct access to
$\inputset_{\text{true}}$, we assume that we only have access to a
corrupted version 
\begin{align}
  \inputset\doteq\lbrace (\V{x}_i, y_i) \rbrace_{i=1}^n
\end{align}
where each $y_i$ is the realization of a random variable $Y$ whose
law $\UKDistri_{Y|X}$ (which is conditioned on $X$) is so that the conditional distribution
$\proba_{Y\sim\UKDistri_{Y|X=\V{x}}}(Y|X=\bfx)$ is fully summarized into a {\em
  known} confusion matrix $\confusion$ given by
\begin{align}
  \forall \V{x},\; \confusion_{pt(\V{x})} \doteq \proba_{Y\sim\UKDistri_{Y|X=\V{x}}}(Y = p |X =
  \bfx)=\proba_{Y\sim\UKDistri_{Y|X=\V{x}}}(Y = p |t(\V{x}) = q).
\label{eq:confusion}
\end{align}
Henceforth, the noise process that corrupts the data is {\em uniform}
within each class and its level does not depend on the precise location of
$\bfx$ within the region that corresponds to class $t(\V{x})$. 
Noise process $Y$ is both very aggressive, as it does not
only apply, as we may expect, to regions close to the boundaries
between classes and very regular,  as the 
mislabelling rate is piecewise constant.

The setting we assume allows us to view $\inputset$ as the
realization of a random sample $\{(X_i,Y_i)\}_{i=1}^n$, where each
pair $(X_i,Y_i)$ is an independent copy of the random pair $(X,Y)$ of
law $\UKDistri_{XY}\doteq\UKDistri_X\UKDistri_{X|Y}.$

\subsection{Problem: Learning a Linear Classifier from Noisy Data}
The problem we address is the learning of a classifier $f$ from
$\inputset$ {\em and} $\confusion$ so that the error
rate $$\risk_{\terror}(f)=\proba_{X\sim\UKDistri}(f(X)\neq t(X))$$ of
$f$, is as small as possible: the usual goal of learning a
classifiier $f$ with small risk is preserved, while now the training data
is only made of corrupted labelled pairs.

Building on Assumption~\ref{ass:linear}, we may refine our learning
objective by restricting ourselves to linear classifiers $f_W$, for
$W=[\V{w}_1\cdots\V{w}_Q]\in\realset^{d\times Q}$ such that 
\begin{equation}
\label{eq:linearprediction}
\forall\bfx\in\inputspace,\;
f_W(\bfx)\doteq\argmax\nolimits_{q\in\classes}\langle\bfw_q,\bfx\rangle,
\end{equation}
and our goal is thus to learn a relevant matrix $W$ from
$\inputset$ {\em and} the confusion information $\confusion$.


\section{{\sc Uma}: Unconfused Ultraconservative Multiclass Algorithm}
\label{sec:uma}

\subsection{Main Result and High Level Justification}
This section presents our main contribution, \uma, a theoretically grounded noise-tolerant
multiclass algorithm depicted in Algorithm~\ref{alg:uma}. \uma learns
and outputs a matrix $W=[\V{w}_1\cdots\V{w}_Q]\in\realset^{d\times Q}$
from a noisy training set $\trainingset$ to produce the associated
classifier
\begin{equation} \label{eq:additive}
f_W(\cdot)=\argmax_q\langle\bfw_q,\cdot\rangle
\end{equation}
by iteratively
updating the $\V{w}_q$'s, whilst maintaining $\sum_q\V{w}_q=0$
throughout the learning process. We may already recognize the 
generic {\em step sizes} promoted by ultraconservative algorithms in
step~\ref{step:uc1} and step~\ref{step:uc2} of the algorithm
\citep{crammer03ultraconservative}.
 An important feature of \uma
is that it only uses information provided by $\trainingset$ and does
not make assumption on the accessibility to the noise-free dataset $\trainingset_{\text{true}}$.

\begin{algorithm}[tbh]
\caption{\uma:  Unconfused Ultraconservative Multiclass Algorithm\label{alg:uma}}
\begin{algorithmic}[1]
\REQUIRE $\inputset= \{(\V{x}_i, y_i)\}_{i=1}^{n}$,
$\confusion\in\realset^{Q\times Q}$, confusion matrix and $\alpha>0$
\ENSURE $W=\left[\V{w}_1, \ldots, \V{w}_K \right]$ and classifier $f_W(\cdot)=\argmax_{q}\langle\bfw_q,\cdot\rangle$\vspace{3mm}
\STATE $\V{w}_k\leftarrow 0$, $\forall k\in\classes$
\REPEAT 
\STATE select $p$ and $q$ \label{step:select}
\STATE $\setA{p}^{\alpha} \leftarrow  \left\{ \V{x} |\V{x}\in\inputset \wedge 
        \dotProd{\V{w}_p}{\V{x}} 
        - \dotProd{\V{w}_k}{\V{x}} 
        > \alpha,\;\forall k\neq p\right\}$\label{step:mistakestart}
\STATE $\gamma_k^p\leftarrow\frac{1}{n}\sum_{i:y_i=k\wedge
        \V{x}_i\in \setA{p}^{\alpha}}\V{x}_i^{\top}$, $\forall k\in\classes$\label{step:gamma_kp}
\STATE form $\Gamma^p\in\realset^{Q\times d}$ as\\
\hspace{2mm}$\Gamma^p\leftarrow\left[\begin{array}{c}\gamma_1^p\\
    \vdots\\ \gamma_Q^p\end{array}\right],$
\STATE  compute the update vector $\xup{pq}$ according to\\
\hspace{2mm}${\V{z}}_{pq}\leftarrow ([\confusion^{-1}\Gamma^p]_q)^{\top}$,\quad(where
$[A]_q$ refers to the $q$th row of matrix $A$) \label{step:mistakeend}
\STATE compute the error set\\
\hspace{2mm}${\cal E}^{\alpha}\leftarrow\{r\in\classes:r\neq q,
\langle\bfw_r,\V{z}_{pq}\rangle-\langle\bfw_q,\V{z}_{pq}\rangle\geq\alpha\}$\label{step:uc1}
\STATE compute some {\em ultraconservative} update steps
$\tau_1,\ldots,\tau_Q$ such that:\\
\hspace{2mm} $\sum_{r=1}^Q\tau_r=0$ and $\left\{\begin{array}{l}\tau_q=1\\ \tau_r\leq 0,\forall
r\in{\cal E}^{\alpha}\\ \tau_r=0,\text{ otherwise}\end{array}\right.$\label{step:uc2}
\STATE perform the updates\\
 \hspace{2mm}$\V{w}_r \leftarrow \V{w}_r +\tau_r \xup{pq}$
\UNTIL{$\|{\bf z}_{pq}\|$ is too small}\label{step:criterion}
\end{algorithmic}
\end{algorithm}

Establishing that under some conditions \uma stops and provides a
classifier with small risk is the purpose of the following
subsections; we will also discuss the unspecified
step~\ref{step:select}, dealing with the selection step.

For the impatient reader, we may already leak some
of the ingredients we use to prove the relevance of our procedure.
The pivotal result regarding the convergence of ultraconservative
algorithms is ultimately a generalized Block-Novikoff
theorem~\citep{crammer03ultraconservative,minsky69perceptrons}, which
  rests on the analysis of the updates made when
training examples are misclassified by the current classifier. If 
the training problem is linearly separable with a positive
margin, then the number of updates/mistakes can be (easily) bounded,
which establishes the convergence of the algorithms. The conveyed
message is therefore that examples that are erred upon are central to
the convergence analysis. It turns out that
step~\ref{step:mistakestart} through~\ref{step:mistakeend} of \uma
(cf. Algorithm~\ref{alg:uma}) construct,
with high probabilty, a point $\xup{pq}$ that is mistaken on. More
precisely, the true class $t(\xup{pq})$ of $\xup{pq}$ is $q$ and it is predicted to be of
class $p$ by the current classifier; at the same time, these
update vectors are guaranteed to realize a positive margin
condition with respect to $W^*$:
$\langle\V{w}_q^*,\xup{pq}\rangle>\langle\V{w}_k^*,\xup{pq}\rangle$
for all $k\neq q$. The ultraconservative feature of the algorithm is
carried by step~\ref{step:uc1} and step~\ref{step:uc2}, which make it
possible to update any prototype vector $\bfw_r$ with $r\neq q$ having
an inner product $\langle\bfw_r,{\bf
  z}_{pq}\rangle$ with ${\bf z}_{pq}$ larger than $\langle\bfw_q,{\bf
  z}_{pq}\rangle$ (which should be the largest if a correct prediction
were made). The reason why we have results `with high
probability' is because the $z_{pq}$'s are (sample-based) estimates
of update vectors known to be of class $q$ but predicted as being of
class $p$, with $p\neq q$;
computing the accuracy of the sample estimates is one of the important
exercises of what follows. A control on the accuracy makes it possible
for us to then establish the convergence of the proposed algorithm.
In order to ease the analysis we conduct, we assume the following. 
\begin{assumption} From now on, we make the assumption that
$\confusion$ is invertible. Investigating learnability under a milder
constraint is something that goes beyond the scope of the present
paper and that we left for future work.
\end{assumption}
From a practical standpoint, it is worth noticing that there are many
situations where the
confusion matrices are diagonally dominant, therefore invertible.

\subsection{${\bf z}_{pq}$ is Probably a Mistake with Positive Margin} \label{sec:UpnAl}
Here, we prove that the update vector $\xup{pq}$ given in
step~\ref{step:mistakeend} is, with high  probability, a point on
which the current classifier errs.

\begin{proposition}
\label{prop:tilde_update}
Let $W=[\V{w}_1\cdots \V{w}_Q]\in\realset^{d\times Q}$ and $\alpha\geq
0$ be fixed.
Let $\setA{p}^{\alpha}$ be defined as in step~\ref{step:mistakestart} of Algorithm~\ref{alg:uma}, i.e:
\begin{equation}
\setA{p}^{\alpha}\doteq \left\{ \V{x} |\V{x}\in\inputset \wedge 
        \dotProd{\V{w}_p}{\V{x}} 
        - \dotProd{\V{w}_k}{\V{x}} 
        > \alpha,\;\forall k\neq p\right\}.
\label{eq:Ap}
\end{equation}
For $k\in\classes$, $p\neq k$, consider the random
variable $\gamma_{k}^p$:
$$\gamma_{k}^p\doteq\frac{1}{n}\sum_{i}\indicator{Y_i=k}\indicator{X_i\in\setA{p}^{\alpha}}X_i^{\top},$$
 ($\gamma_{k}^p$ of step~\ref{step:gamma_kp} of Algorithm~\ref{alg:uma} is a realization
of this variable, hence the overloading of notation $\gamma_{k}^p$).

The following holds, for all $k\in\classes$:
\begin{equation}
\expectation_{\trainingset}\left\{\gamma_{k}^p\right\}=\expectation_{\{(X_i,Y_i)\}_{i=1}^n}\left\{\gamma_{k}^p\right\}=\sum_{q=1}^Q\confusion_{kq}\mu_q^p,
\end{equation}
where
\begin{equation}\mu_q^p\doteq\expectation_{X}\left\{\indicator{t(X)=q}\indicator{X\in\setA{p}^{\alpha}}X^{\top}\right\}.
\end{equation}
\end{proposition}
\begin{proof}
Let us compute
$\expectation_{XY}\{\indicator{Y=k}\indicator{X\in\setA{p}^{\alpha}}X^{\top}\}$:
\begin{alignsize*}{\small}
\expectation_{XY}\{\indicator{Y=k}\indicator{X\in\setA{p}^{\alpha}}X^{\top}\}&=\int_{\inputspace}\sum_{q=1}^Q\indicator{q=k}\indicator{\V{x}\in\setA{p}^{\alpha}}\V{x}^{\top}\proba_{Y}(Y=q|X=\V{x})d\UKDistri_{X}(\V{x})\\
&=\int_{\inputspace}\indicator{\V{x}\in\setA{p}^{\alpha}}\V{x}^{\top}\proba_{Y}(Y=k|X=\V{x})d\UKDistri_{X}(\V{x})\\
&=\int_{\inputspace}\indicator{\V{x}\in\setA{p}^{\alpha}}\V{x}^{\top}\confusion_{kt(\bfx)}d\UKDistri_{X}(\V{x})\tag{cf.~\eqref{eq:confusion}}\\
&=\int_{\inputspace}\sum_{q=1}^Q\indicator{t(\V{x})=q}\indicator{\V{x}\in\setA{p}^{\alpha}}\V{x}^{\top}\confusion_{kq}d\UKDistri_{X}(\V{x})\\
&=\sum_{q=1}^Q \confusion_{kq}\int_{\inputspace}\indicator{t(\V{x})=q}\indicator{\V{x}\in\setA{p}^{\alpha}}\V{x}^{\top}d\UKDistri_{X}(\V{x})=\sum_{q=1}^Q \confusion_{kq}\mu_q^p,
\end{alignsize*}
where the next-to-last line comes from the fact that the classes are
non-overlapping. The fact that the $n$ pairs $(X_i,Y_i)$ are identically
and independently distributed give the result.
\end{proof}

Intuitively, $\mu_q^p$  must be seen as an example of class $p$ which is erroneously predicted
as being of class $q$. Such an example 
is precisely what we are looking for to update the current classifier;
as expecations cannot be computed, the estimate $\xup{pq}$  of
$\mu_q^p$ is used instead of $\mu_q^p$. 

\begin{proposition} \label{prop:tilde_update2}
Let $W=[\V{w}_1\cdots \V{w}_Q]\in\realset^{d\times Q}$ and $\alpha\geq
0$ be fixed.
For $p,q\in\classes$, $p\neq q$, ${\V{z}}_{pq}\in\realset^d$ is
such that
\begin{align}
&\expectation_{XY}{{\V{z}}_{pq}}=\mu_{q}^p\\
&\langle\V{w}_q^*,\mu_{q}^p\rangle -
\langle\V{w}_k^*,\mu_{q}^p\rangle\geq\theta,\;\forall
k\neq q,\label{eq:mumargin}\\
&\langle\V{w}_p,\mu_{q}^p\rangle -
\langle\V{w}_k,\mu_{q}^p\rangle> \alpha,\;\forall k\neq p.\label{eq:muerror}
\end{align}
(Normally, we should consider the transpose of $\mu_{q}^p$, but since
we deal with vectors of $\realset^d$ ---and not matrices--- we omit
the transpose
for sake of readability.)

This means that 
\begin{enumerate}[i)]
\item $t(\mu_{q}^p)=q$, i.e. the `true' class of $\mu_{q}^p$ is $q$;
\item  and $f_W(\mu_{q}^p)=p$: $\mu_{q}^p$ 
is therefore misclassified by the current classifier $f_W$.
\end{enumerate}
\end{proposition}
\begin{proof}
According to Proposition~\ref{prop:tilde_update},
$$\expectation_{XY}\left\{\Gamma^{p}\right\}=\expectation_{XY}\left\{\left[\begin{array}{c}\gamma_1^p\\
    \vdots\\ \gamma_Q^p\end{array}\right]\right\}=\left[\begin{array}{c}\expectation_{XY}\left\{\gamma_1^p\right\}\\
    \vdots\\
    \expectation_{XY}\left\{\gamma_Q^p\right\}\end{array}\right]=\left[\begin{array}{c}\sum_{q=1}^Q\confusion_{1q}\mu_q^p\\\vdots\\\sum_{q=1}^Q\confusion_{Qq}\mu_q^p\end{array}\right]=\confusion
\left[\begin{array}{c}\mu_1^p\\\vdots\\\mu_Q^p\end{array}\right].$$

Hence, inverting $\confusion$ and extracting the $q$th of the
resulting matrix equality gives that $\expectation\left\{{\V{z}}_{pq}\right\}=\mu_{q}^p$.

Equation~\eqref{eq:mumargin} is obtained thanks to
Assumption~\ref{ass:linear} combined with  the linearity of the
expectation. Equation~\eqref{eq:muerror} is obtained thanks to the
definition~\eqref{eq:Ap} of $\setA{p}^{\alpha}$ (made of points that are predicted to be of
class $p$)  and the linearity of the expectation.
\end{proof}

\begin{proposition}
\label{prop:zpqerror}
Let $\varepsilon>0$ and $\delta\in(0;1]$.
There exists a number $$
 n_0(\epsilon, \delta, d, Q) = O\left( \frac{1}{\epsilon^2} \left[
\ln\frac{1}{\delta} + \ln Q + d\ln\frac{1}{\epsilon}
\right] \right) $$
 such that if the number of training
samples is greater than $n_0$ then, with high probability
\begin{align}
&\langle\V{w}_q^*,\xup{pq}\rangle -
\langle\V{w}_k^*,\xup{pq}\rangle 
\geq \theta - \epsilon
\label{eq:zmargin} \\
&\langle\V{w}_p,\xup{pq}\rangle -
\langle\V{w}_k,\xup{pq}\rangle\geq 0,\;\forall k\neq p.\label{eq:zerror}
\end{align}
\end{proposition}
\begin{proof} 
The existence of $n_0$ rely on 
pseudo-dimension arguments.
 We defer this
part of the proof to Appendix~\ref{apd:first} and
we will directly assume here that if $n \geq n_0$,
then, with probability $1 - \delta$ :

\begin{align}
  | \dotProd{\V{w}_p - \V{w}_q}{\xup{pq}} -
  \dotProd{\V{w}_p - \V{w}_q}{\mupt{q}{p}} | \leq
 \varepsilon \label{eq:zproof}
\end{align}
for any $\V{W}$, $\xup{pq}$.

Proving \eqref{eq:zmargin} then proceed by observing that
\begin{align*}
\dotProd{\V{w}_q^* - \V{w}_k^*}{\xup{pq}} = \dotProd{\V{w}_q^* - \V{w}_k^*}{\mupt{q}{p}}
+ \dotProd{\V{w}_q^* - \V{w}_k^*}{\xup{pq} - \mupt{q}{p}}
\end{align*}
bounding the first part using Proposition $2$:
$$\dotProd{\V{w}_q^* - \V{w}_k^*}{\mupt{q}{p}} \geq \theta
$$
and the second one with \eqref{eq:zproof}.
A similar reasoning allows us to get \eqref{eq:zerror} by setting
$\alpha \doteq \varepsilon$ in $\setA{p}^{\alpha}$ .
\end{proof}
This last proposition essentially says that the update vectors ${\bf
  z}_{pq}$ that we compute are, with high probability, erred upon and
realize a margin condition $\theta/2.$

\subsection{Convergence and Stopping Criterion}
In order to arrive at the main result of our paper, we recall the
following Theorem of~\cite{crammer03ultraconservative}, which stated
the convergence of general ultraconservative learning algorithms.

\begin{theorem}[\cite{crammer03ultraconservative}] \label{Th:MultiNovikoff}
  Let $\{({\bfx}_i, t_i)\}_{i=1}^n$ be a linearly separable sample  such that there
  exists a separating linear classifier $W^*$ with margin
  $\margin$. Then {\em any}
  multiclass ultraconservative algorithm will make no more than $2 /
  {\margin}^2$ updates before convergence.
\end{theorem}

We arrive at our main result, which provides both
convergence and a stopping criterion. 
\begin{proposition}
\label{prop:convergence}
There exists a number $n_1$, polynomial in $d, 1/\H{\theta}, Q, 1/\delta$, such that
is the training sample is of size at least $n_1$, then, with high
probability ($1 - \delta$), \uma makes at most $O(1/\H{\theta}^2)$ 
 iterations for all $p,q$.
\end{proposition}
\begin{proof}
This results is obtained as a direct combination of
Proposition~\ref{prop:zpqerror} and Theorem~\ref{Th:MultiNovikoff},
with the adjusted margin $\theta - \varepsilon$ It comes as
in~\cite{blum96polynomialtime} that the conditional misclassification errors
$\proba(f_{W}(X)=p|Y=q)$ are all small.
\end{proof}

\subsection{Selecting $p$ and $q$} \label{sec:pqselect}

So far, the question of 
selecting good pairs of values $p$ and $q$ to perform updates has been
left unanswered.  
Intuitively, we would like to focus on the pair $(p,q)$ for which the instantaneous
empirical misclassification rate is the highest. We want to privilege those pairs $(p,q)$
because, first the induced update will
likely lead to a greater reduction of the error and  then, and more
importantly, because $\xup{pq}$ will be more reliable, as computed on
larger samples of data.
Another wise strategy for selecting $p$ and $q$ 
might be to select the pair $(p,q)$ for which the
empirical probability $\hat{\proba}(f_W(X)=p|Y=q)$ is the largest.

From these two strategies, we propose two possible computations
for $(p,q)$:
\begin{align}
  &(p,q)_{\terror} \doteq \arg\max_{(p,q)} \|{\bf z}_{pq}\|\\
  &(p,q)_{\tconf} \doteq \arg\max_{(p,q)}\frac{\|{\bf
      z}_{pq}\|}{\hat{\pi}_q}
\end{align}
where $\hat{\pi}_q$ is the estimated proportion of examples of true
class $q$ in the training sample. In a way similar to the computation of
${\bf z}_{pq}$ in Algorithm~\ref{alg:uma}, $\hat{\pi}_q$ may be
computed as follows:
$$\hat{\pi}_q=\frac{1}{n}[\confusion^{-1}{\hat{\bfy}}]_q,$$
where $\bfy\in\realset^Q$ is the vector containing the number of
examples from $\inputset$ having noisy labels $1,\ldots,Q$, respectively.

The second selection criterion is intended to normalize the number of
errors with respect to the proportions of different classes and aims at
being robust to imbalanced data.



\section{Experiments} \label{sec:Expe}

In this section, we present numerical results from empirical
evaluations of our approach and we
discuss different practical aspects of \uma. The ultraconservative
step sizes retained are those corresponding to a regular Percepton,
or, in short: $\tau_p=-1$ and $\tau_q=+1$.

Section \ref{sec:expe:toy} covers robustness results, based
on simulations with synthetic data while 
Section \ref{sec:expe:reuters} takes it a step further and
evaluates our algorithm on real data, with a realistic
noise process related to Example~\ref{ex:redwire} (cf. Section~\label{sec:intro}).

We essentially use the confusion rate as a performance measure,
which is the Frobenius norm of the confusion matrix estimated on 
a test set $S_{\text{test}}$  (independent from the training set)  and which is defined as: $$\fro{\H{\Conf}} = \sqrt{\sum_{i,j}
\H{\Conf}_{ij}^2},\text{ with }\H{\Conf}_{pq} = 
\frac{\sum_{\V{x}_i \in S_{\text{test}}} \indicator{t_i = q \wedge \H{y}_i = q}}
{\sum_{\V{x}_i \in S_{\text{test}}} \indicator{t_i = q}},$$
where $\widehat{y}_i$ is the label predicted for test instance ${\bf
  x}_i$ by the learned predictor.

\subsection{Toy dataset} \label{sec:expe:toy}

We use a $10$-class dataset with a total of roughly $1,000$
$2$-dimensional examples uniformly distributed over the
unit circle centered at the origin.
Labelling is achieved according to \eqref{eq:linearprediction},
where each of the 10 $\V{w}$ are randomly generated from the same
distribution than the examples (uniformly over the unit circle
centered at the origin).
Then a margin
$\margin = 0.025$ is enforced by removing examples
too close of the decision boundaries. 
Note that because of the way we generate them, all
the normal
vector $\V{w}$ of each hyperplanes are normalized. This
is  to prevent the cases
where the margin with respect to some given class is artificially high because
of the norm of $\V{w}$. Hence, 
all classes are guaranteed to be present in the training set as long
as the margin constraint does not preclude it\footnote{Practically, the case where three classes are so close to
  each other that
no examples can lie in the middle one because of the margin constraint
never occurred with such small margin.}.

The learned classifiers are tested against a similarly generated dataset
of $10,000$ points. The results reported in the tables are averaged over $10$ runs.

The noise is generated from the sole confusion
matrix. This
situation can be tough to handle and is rarely met with real
data but we stick with it as it is a good example of a
worst-case scenario.

We first evaluate the robustness to noise level
(Fig.~\ref{fig:noise_robustness}).
We randomly generate a reference stochastic matrix $M$ and
we define $N$ such that $M = I + 10 \times N$.
Then we run \uma $20$ times with
confusion matrix $\Conf$ ranging from $N$ to $20N$.
Figure~\ref{fig:noise_robustness} plots the confusion rate
against $\fro{\Conf}$.

The second experiment (Fig.~\ref{fig:noise_estimation}) evaluates the robustness
to errors in the estimation of the noise. We proceed in a 
similar way as before, but the data are always corrupted according to
$M$. The \emph{approximation} factor of $\Conf = i \times N$ 
is then defined as $1 - i/10$.  Figure~\ref{fig:noise_estimation} plots the confusion
rate against this approximation factor. 
Note that an approximation factor
of $1$ corresponds to the identity matrix and, in that case, \uma behaves
as a regular Perceptron. More generally, the noise is underestimated
when the approximation factor is positive, and overestimated when it is
negative.

\begin{figure}
  \centering
  \subfigure[Robustness to noise\label{fig:noise_robustness}]
{
    \includegraphics[width=0.40\textwidth]{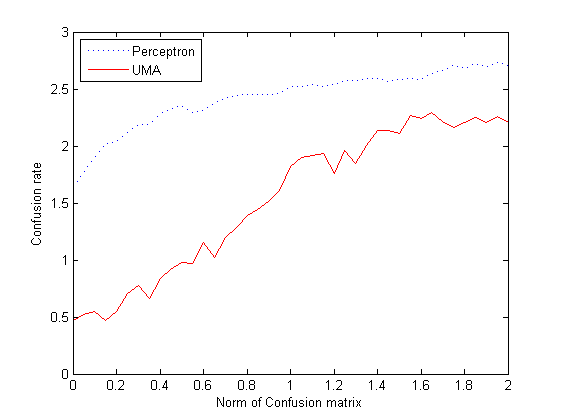}
} 
\subfigure[Robustness to noise estimation\label{fig:noise_estimation}]
{
    \includegraphics[width=0.40\textwidth]{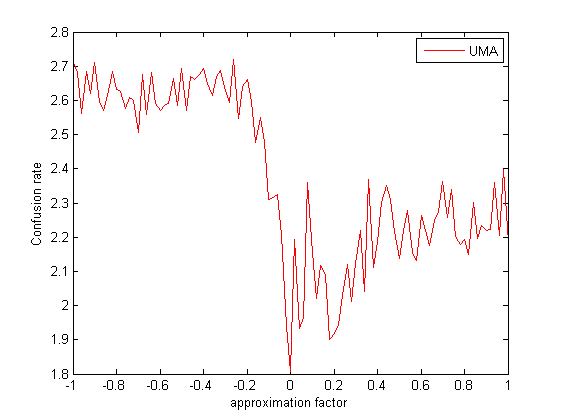}
}
  \caption{ (a) evolution of the Confusion rate 
for different noise
    levels; (b) evolution of the same quantity with 
respect to errors in the confusion matrix $\Conf$.}
\end{figure}

These first simulations show that \uma provides improvements
over the Perceptron for every noise level tested ---although its
performance obviously slightly degrades as the noise level increases. The second simulation
points out that, in addition to being robust to the noise process
itself,  \uma is also robust to underestimated
noise levels, but not to overestimated ones.

\subsection{Reuters} \label{sec:expe:reuters}

The Reuters dataset is a nearly linearly separable document categorization
dataset. The training (resp. test) set is made of
approximately $15,000$ (resp. $300,000$) examples in 
roughly $50,000$ dimensions spread over $50$ classes.

To put ourselves in a realistic learning scenario, we assume 
that labelling examples is
very expensive and we implement the strategy evoked in
Example~\ref{ex:redwire} (cf. introduction). Thus, we label only $3$ percent
(selected at random) of
the training set. Our strategy is to
learn a classifier over half of the labelled
examples, and then to use the classifier to label the entire training set. This way, we end up with a labelled training set 
with noisy labels, since $1.5$ percent of the whole training set is
evidently not sufficient to learn a reliable classifier. In order to
gather the information needed by \uma to learn, we estimate the confusion
matrix of the whole training set using the remaining 1.5\% correctly
labelled examples (i.e. those examples not used for learning the
labelling classifier).

However, it occurs that some classes are so under-represented that
they are flooded by the noise process and/or are 
not present in the original $1.5$ percent, which leads to 
a non-invertible confusion matrix. We therefore
restrict the dataset to the $19$ largest classes. One might
wonder whether doing so removes class imbalance. This is not
the case as the least represented class accounts for $198$ examples
while this number reaches $4,000$ for the most represented one.
As in Section~\ref{sec:expe:toy}, the results reported are an average over 10 runs
of the same experiment.

We use three variations of \uma with different strategies
for selecting $(p,q)$ (error, confusion, and random) and monitor
each one along the learning process. 

We also compare \uma with the regular Perceptron in order
to quantify the increase in  accuracy induced by the use of the
confusion matrix. We compare \uma with two different
settings, one with the noisy classes ({\tt Mperc}), 
and one with the true data ({\tt Mperc$_{\text{full}}$}). 

\begin{figure}[h]
  \includegraphics[scale=0.5]{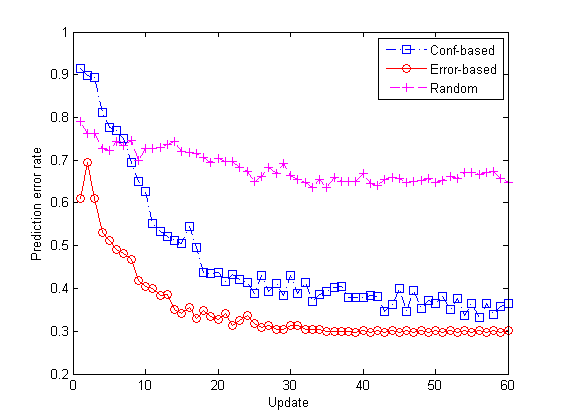}
  \includegraphics[scale=0.5]{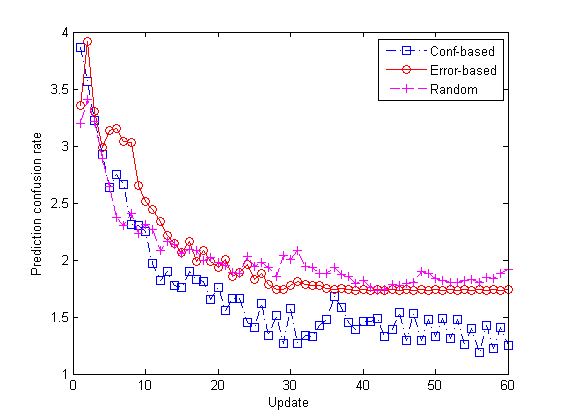}
  \caption{Error and confusion risk on Reuters dataset
  with various update strategies.}
  \label{fig:f3}
\end{figure}

From Figure \ref{fig:f3}, we observe that  both performance measures evolve
similarly, attaining a stable state around the $30^{\text{th}}$ 
iteration. The best strategy depends of the
performance measure used, even though it is noticeable the random strategy
is alway non-optimal.

As one might expect, the confusion-based strategy
performs better than the error-based one
 with respect to the confusion rate
while the converse holds when considering the error rate. This
observation motivates us to thoroughly study the confusion-based
strategy in a near future.

The plateau reached around the $30^{\text{th}}$ may be puzzling,
since the studied dataset presents
no positive margin and convergence
is therefore not guaranteed.
However the non-existence of a linear separability  can be also
interpreted as label noise. Keeping
this in mind, our current setup is equivalent to the
one where noise is added after the estimation
of the confusion matrix, thus leading to an
underestimated confusion matrix. Nonetheless
it is worth noting that, according to Figure
 \ref{fig:noise_estimation}, \uma is robust to 
this situation, which explains the good results exhibited.

\begin{table}
  \centering
  \begin{tabular}{ccccc}
  \toprule
  Algorithm & {\tt M-perc} & {\tt M-perc$_{\text{full}}$} & \uma\\
  \midrule
  error rate & $0.482$ & $0.379$ & $0.381$ \\
  \bottomrule
  \end{tabular}
  \caption{Performances of different algorithms}
  \label{fig:tab}
\end{table}

Looking at  the results, it is clear that \uma successfully
recovers from the noise, attaining an error rate similar to
the one attained with the uncorrupted data (cf. Table~\ref{fig:tab}). However
it is good to temper this result as Reuters is not linearly
separable. We have already seen 
(Fig. \ref{fig:noise_estimation}, Figure~\ref{fig:f3})
that this is not really a problem for \uma, unlike the
regular multiclass Perceptron. Thus, the present setting,
while providing a great practical setup, is somewhat
biaised in favor of \uma.

Finally, note that the results of Section~\ref{sec:expe:reuters}
are quite better than those of Section~\ref{sec:expe:toy} .
Indeed, in the Reuters experiment, the noise is concentrated
around the decision boundaries because this is typically where
the labelling classifier would make errors. This kind of noise is
called monotonic noise and is a special, easier, case of 
 classification noise as discussed in 
\cite{bylander94learning}.



\section{conclusion} \label{sec:Concl}

In this paper, we have proposed a new algorithm to cope
with noisy examples in multiclass linear problems. 
This is, to the best of our knowledge, the first time
the confusion matrix is used as a way to handle noisy label
in multiclass problem.
Moreover, \uma does not consist in a binary mapping of
the original problem and therefore benefits from the good
theoretical guarantees of additive algorithms.
Besides, \uma can be adapted to a wide variety of 
additive algorithms, allowing to handle noise with
multiple methods.
On the sample complexity side, 
note that a more tight bound can
be derived with specific multiclass tools as
the Natarajan's dimension (see \cite{DanielySBS11} 
for example). However as it is not the core of this
paper we stuck with the classic analysis here.

To complement this work, we want to investigate
a way to properly tackle near-linear problems (as Reuters). 
As for now the algorithm already does a very good jobs thanks to its
noise robustness. However more work has to be done to derive a
proper way to handle case where a perfect classifier does not exists.
We think there
are great avenues for interesting research in this domain with an algorithm
like \uma and we are curious to see how this present work may carry
over to more general problems.

\bibliography{unconfused}

\appendix

\section{Double sample theorem}\label{apd:first}

\begin{proof}[Proposition \ref{prop:zpqerror}]

For a fixed couple $(p,q) \in \outputspace^2$  we consider
the family of functions
$$ \bigF_{pq} \doteq \lbrace f : f(\V{x}) \doteq 
\dotProd{\V{w}_p - \V{w}_q}{x} : \V{w}_p, \V{w}_q \in \realset^d
\rbrace $$ 
with the corresponding loss function $$l^{f}(\V{x}, y) \doteq l(f(\V{x}), y)
\doteq 1 + f(\V{x})$$

Clearly, $\bigF_{pq}$ is a subspace of affine functions, thus 
$\Pdim(\bigF_{pq}) \leq (n + 1)$, where $\Pdim(\bigF_{pq})$ is the
pseudo-dimension of $\bigF_{pq}$. Additionally, $l$ is Lipschitz
in its first argument with a Lipschitz factor of $L \doteq 1$ as
$\forall y, y_1, y_2, \in \outputspace :
 \vert l(y_1, y) - l(y_2, y) \vert = \vert y_1 - y_2 \vert$

Let $\UKDistri_{pq}$ be any distribution over $\inputspace \times
\outputspace$ and $T \in (\inputspace \times \outputspace)^m$ 
such that $T \sim \UKDistri_{pq}^m$, then define the 
\emph{empirical loss} $\err_T^l[f] \doteq \frac{1}{m}
\sum_{\V{x}_i \in T} l(\V{x}_i, y_i)$ and the
\emph{expected loss} $\err_{\UKDistri}^l [f] \doteq
 \Ex{\UKDistri}{l(\V{x},y)}$

The goal here is to prove that

\begin{align} \label{eq:result}
  \proba_{T \sim \UKDistri-{pq}^m} \left(
\sup_{f \in \bigF_{pq}} \vert \err_{\UKDistri}^l [f]
  - \err_{T}^l [f] \geq \epsilon
 \right) \leq O \left( 4 \times
 \left( \frac{8}{\epsilon} \right)^{(d+1)}
e^{ {m\epsilon^2}/_{128} } \right)
\end{align}

\begin{proof}[Proof of \eqref{eq:result}]
We start by noting that $l(y_1, y_2) \in [0, 2]$ and then proceed
with a classic $4$-step double sampling proof. Namely :

{\bf Symmetrization}
We introduce a \emph{ghost} sample $T' \in (\inputspace \times
\outputspace)^m$, $T' \sim \UKDistri_{pq}^m$ and show that for
$f^{\text{bad}}_T$ such that $ \vert \err_{\UKDistri_{pq}}^l
[f^{\text{bad}}_T] - \err_T^l [f^{\text{bad}}_T] \vert \geq \epsilon
$ then 

$$ \proba_{T' \vert T} \left(
\left| \err_{T'}^l [f^{\text{bad}}_T] - \err_{\UKDistri_{pq}}^l 
[f^{\text{bad}}_T] \right| \leq \frac{\epsilon}{2}
\right) \geq \frac{1}{2},$$

as long as $m\epsilon^2 \geq 32.$

It follows that 

\begin{align*}
& \proba_{(T,T') \sim \UKDistri_{pq}^m \times \UKDistri_{pq}^m}
\left( \sup_{f \in \bigF_{pq}} \vert \err_T^l [f] - \err_{T'}^l [f]
  \vert \geq \frac{\epsilon}{2}
\right) \\
&~ \geq \proba_{T \sim \UKDistri_{pq}^m} 
\left( \vert \err_{T}^l [f^{\text{bad}}_T] - \err_{\UKDistri_{pq}}^l
  [f^{\text{bad}}_T] \vert \geq \epsilon
\right) \times
\proba_{T' \vert T} \left(
\left| \err_{T'}^l [f^{\text{bad}}_T] - \err_{\UKDistri_{pq}}^l 
[f^{\text{bad}}_T] \right| \leq \frac{\epsilon}{2}
\right) \\
&~  = \frac{1}{2} \proba_{T \sim \UKDistri_{pq}^m} 
\left( \vert \err_{T}^l [f^{\text{bad}}_T] - \err_{\UKDistri_{pq}}^l
  [f^{\text{bad}}_T] \vert \geq \epsilon
\right) \\
&~  = \frac{1}{2} \proba_{T \sim \UKDistri_{pq}^m} 
\left( \sup_{f \in \bigF_{pq}} \vert \err_{T}^l [f] - \err_{\UKDistri_{pq}}^l
  [f] \vert \geq \epsilon
\right) \tag{By definition of $f^{\text{bad}}_T$}
\end{align*}

Thus upper bounding the desired probability by

\begin{align}
2 \times \proba_{(T,T') \sim \UKDistri_{pq}^m \times \UKDistri_{pq}^m}
\left( \sup_{f \in \bigF_{pq}} \vert \err_T^l [f] - \err_{T'}^l [f]
  \vert \geq \frac{\epsilon}{2}
\right) \label{eq:sndProba}
\end{align}

{\bf Swapping Permutation}
Let define $\Gamma_{m}$ the set of all
permutations that swap one or more elements of
$T$ with the corresponding element of $T'$ (i.e.
the $i$th element of $T$ is swapped with 
the $i$th element of $T'$). It is quite 
immediate that $\vert \Gamma_{m} \vert = 2^m$.
For each permutation $\sigma \in \Gamma_m$ we
note $\sigma(T)$ (resp. $\sigma(T')$) the set
originating from $T$ (resp. $T'$) from which
the elements have been swapped with $T'$
(resp. $T$) according to $\sigma$.

Thanks to $\Gamma_m$ we will be able to provide an
upper bound on \eqref{eq:sndProba}. Our starting
point is since $(T,T') \sim \UKDistri_{pq}^m \times
\UKDistri_{pq}^m$ then for any $\sigma \in \Gamma_m$, 
the random variable
$\sup_{f \in \bigF_{pq}} \vert \err_T^l [f] - \err_{T'}^l [f] \vert$
follows the same distribution as
$\sup_{f \in \bigF_{pq}} \vert \err_{\sigma(T)}^l [f] -
\err_{\sigma(T')}^l [f] \vert$.

Therefore :

\begin{align}
& \proba_{(T,T') \sim \UKDistri_{pq}^m \times \UKDistri_{pq}^m}
\left( \sup_{f \in \bigF_{pq}} \vert \err_T^l [f] - \err_{T'}^l [f]
  \vert \geq \frac{\epsilon}{2}
\right) \notag \\
& ~ = \frac{1}{2m} \sum_{\sigma \in \Gamma_m}
\proba_{T,T' \sim \UKDistri_{pq}^m \times \UKDistri_{pq}^m} \left(
\sup_{f \in \bigF_{pq}} \vert \err_{\sigma(T)}^l [f] - 
\err_{\sigma(T')}^l [f] \vert \geq \frac{\epsilon}{2}
\right) \notag \\
&~ = \Ex{(T,T') \sim \UKDistri_{pq}^m \times \UKDistri_{pq}^m}
{\frac{1}{2m} \sum_{\sigma \in \Gamma_m} \indicator{
\sup_{f \in \bigF_{pq}} \vert \err_{\sigma(T)}^l [f] -
\err_{\sigma(T')}^l [f] \vert \geq \frac{\epsilon}{2}
}} \notag \\
& ~ \leq \sup_{(T,T') \in (\inputspace \times \outputspace)^{2m}}
\left[
\proba_{\sigma \in \Gamma_m} \left(
\sup_{f \in \bigF_{pq}} \vert \err_{\sigma(T)}^l [f] -
\err_{\sigma(T')}^l [f] \vert \geq \frac{\epsilon}{2}
\right)
\right] \label{eq:thdProba}
\end{align}

And this concludes the second step.

{\bf Reduction to a finite class}
The idea is to reduce $\bigF_{pq}$ in
\eqref{eq:thdProba} to a finite class of functions.
For the sake of conciseness, we will not enter
into the details of the theory of \emph{covering numbers}.
Please refer to the corresponding literature for 
further details (eg. \cite{DevroyeGL}).

In the following, $\bigN({\epsilon}/_8, \bigF_{pq}, 2m)$ will denote the
\emph{uniform ${\epsilon}/_8$convering number} of $\bigF_{pq}$
over a sample of size $2m$.

Let $\bigG_{pq} \subset \bigF_{pq}$ such that $(l^{\bigG_{pq}})_{\vert
(T,T')}$ is an ${\epsilon/_8}$-cover of $(l^{\bigF_{pq}})_{\vert
(T,T')}$.
Thus, $\vert \bigG_{pq} /vert \leq \bigN( {\epsilon}/_8,
l^{\bigF_{pq}}, 2m ) < \infty$ Therefore, if $\exists f \in \bigF_{pq}$
such that $\vert \err_{\sigma(T)}^l [f] - \err_{\sigma(T')}^l [f]
\vert \geq \frac{\epsilon}{2}$ then, $\exists g \in \bigG_{pq}$
such that $\vert \err_{\sigma(T)}^l [g] - \err_{\sigma(T')}^l [g]
\vert \geq \frac{\epsilon}{4}$ and the following comes
naturally

\begin{align}
& \proba_{\sigma \in \Gamma_m} \left(
\sup_{f \in \bigF_{pq}} \vert \err_{\sigma(T)}^l [f] -
\err_{\sigma(T')}^l [f] \vert \geq \frac{\epsilon}{2}
\right) \notag \\
& ~ \leq \proba_{\sigma \in \Gamma_m} \left(
\max_{g \in \bigG_{pq}} \vert \err_{\sigma(T)}^l [g] -
\err_{\sigma(T')}^l [g] \vert \geq \frac{\epsilon}{4}
\right) \notag \\
& ~ \leq \bigN( {\epsilon}/_8, l^{\bigF_{pq}}, 2m ) \max_{g \in \bigG_{pq}}
\proba_{\sigma \in \Gamma_m} \left(
\vert \err_{\sigma(T)}^l [g] - \err_{\sigma(T')}^l [g] \vert 
\geq \frac{\epsilon}{8}
\right) \tag{By union Bound}
\end{align}

{\bf Hoeffding's inequality}

Finally, consider $\vert \err_{\sigma(T)}^l [g] - 
\err_{\sigma(T')}^l [g] \vert$ as the average of $m$
realization of the same random variable, with expectation
equal to $0$. Then by Hoeffding's inequality we have that

\begin{align}
 & \proba_{\sigma \in \Gamma_m} \left(
\vert \err_{\sigma(T)}^l [g] - 
\err_{\sigma(T')}^l [g] \vert \geq \frac{\epsilon}{4}
\right) \leq 2 e^{{-m\epsilon^2}/_{128}} \label{eq:final}
\end{align}

Putting everything together holds the result w.r.t.
$\bigN ({\epsilon}/_8, l^{\bigF_{pq}}, 2m)$ for
$m\epsilon^2 \geq 32$. For $m\epsilon^2 < 32$
it holds trivially.

Remind that $l^{\bigF_{pq}}$ is Lipschitz in its first
argument with a Lipschitz constant $L = 1$ thus
$\bigN ({\epsilon}/_8, l^{\bigF_{pq}}, 2m) \leq 
\bigN ({\epsilon}/_8, \bigF_{pq}, 2m) = O \left(
 \left( \frac{8}{\epsilon} \right)^{\Pdim(\bigF_{pq})} \right)$
\end{proof}

Let us now consider a slightly modified
definition of $\bigF_{pq}$ :

$$ \H{\bigF_{pq}} \doteq \lbrace f : f(\V{x}) \doteq 
\indicator{t(\V{x}) = q} \indicator{\V{x} = \setA{p}^{\alpha}}
\dotProd{\V{w}_p - \V{w}_q}{x} : \V{w}_p, \V{w}_q \in \realset^d
\rbrace $$ 

Clearly, for any fixed $(p,q)$ the same result holds as we
never use any specific information about $\bigF_{pq}$,
except its pseudo dimension which is untouched.
Indeed, for each function in $\H{\bigF_{pq}}$ there is
at most one corresponding affine function.

It come naturally that, fixing $S$ as the training set,
the following holds true :
$$\frac{1}{m} \sum_m  
\indicator{t(\V{x}) = q} \indicator{\V{x} = \setA{p}^{\alpha}}
\V{x} =
\V{z}_{pq}.$$ Thus $$\left| \err_T^l [f] - \err_D^l [f] 
\right| = \left| \dotProd{\frac{\V{w}_p - \V{w}_q}
{\norm{\V{w}_p - \V{w}_q}}}{\V{z}_{pq}} - 
\dotProd{\frac{\V{w}_p - \V{w}_q}
{\norm{\V{w}_p - \V{w}_q}}}{\mupt{p}{q}} \right|.$$

We can generalize this result for any couple $(p,q)$
by a simple union bound, giving the desired 
inequality:

\begin{align*}\proba_{(\inputspace \times \outputspace) \sim \UKDistri} &\left(
\sup_{W \in \realset^{d \times Q}} \left| \dotProd{\frac{\V{w}_p - \V{w}_q}
{\norm{\V{w}_p - \V{w}_q}}}{\V{z}_{pq}} - 
\dotProd{\frac{\V{w}_p - \V{w}_q}
{\norm{\V{w}_p - \V{w}_q}}}{\mupt{p}{q}} \right|\geq \epsilon
\right) \\
 &\qquad\leq  O \left( 4 Q^2 \times
 \left( \frac{8}{\epsilon} \right)^{(n+1)}
e^{ {m\epsilon^2}/_{128} } \right)
\end{align*}

Equivalently, we have that

$$ \left| \dotProd{\frac{\V{w}_p - \V{w}_q}
{\norm{\V{w}_p - \V{w}_q}}}{\V{z}_{pq}} - 
\dotProd{\frac{\V{w}_p - \V{w}_q}
{\norm{\V{w}_p - \V{w}_q}}}{\mupt{p}{q}} \right| \geq \epsilon $$

with probability $1 - \delta$ for 

$$ m \geq O\left( \frac{1}{\epsilon^2} \left[
\ln\left( \frac{1}{\delta} \right) + \ln(Q) + d\ln\left( \frac{1}{\epsilon} \right)
\right] \right) $$

\end{proof}

\end{document}